\documentclass[lettersize,journal]{IEEEtran}

\usepackage[T1]{fontenc}
\usepackage[utf8]{inputenc}

\usepackage{amsmath,amsfonts,amsthm}

\newtheorem{theorem}{Theorem}[section]

\newtheorem{assumption}{Assumption}[section]
\newtheorem{definition}{Definition}[section]
\newtheorem{proposition}{Proposition}[section]

\newtheorem{remark}{Remark}

\usepackage{graphicx}
\usepackage{array}
\usepackage{multirow}
\usepackage{stfloats}
\usepackage{caption}
\usepackage{subcaption}

\renewcommand{\arraystretch}{1.5}
\usepackage{tikz}
\usetikzlibrary{positioning}
\usepackage{pgfplots}
\usepackage{pgfplotstable}
\pgfplotsset{compat=newest}   

\usepackage[ruled,lined,linesnumbered,commentsnumbered,longend]{algorithm2e}
\usepackage{algorithmicx}
\usepackage{algpseudocode}

\SetKwInput{KwInput}{Input}
\SetKwInput{KwOutput}{Output}

\usepackage{cite}
\usepackage{url}
\usepackage[colorlinks=true, allcolors=blue]{hyperref}  

\newcommand{\orcidicon}{%
    \includegraphics[width=0.32cm]{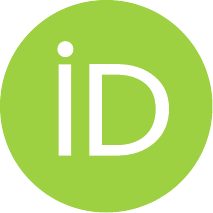}}

\foreach \x in {A,...,Z}{%
    \expandafter\xdef\csname orcid\x\endcsname{%
        \noexpand\href{https://orcid.org/\csname orcidauthor\x\endcsname}%
                      {\noexpand\orcidicon}}}

\usepackage{xcolor}
\usepackage{verbatim}
\usepackage{color,soul}
\usepackage{textcomp}
\usepackage{balance}

\def\BibTeX{{\rm B\kern-.05em{\sc i\kern-.025em b}%
             \kern-.08em T\kern-.1667em\lower.7ex\hbox{E}%
             \kern-.125emX}}

\begin{document}

\title{Entropy-Guided Tensor Compression for Multimodal Federated Learning on Edge Devices} 
\newcommand{\orcidauthorA}{0009-0002-5643-3828} 
\newcommand{\orcidauthorD}{0000-0002-9485-7720} 

\author{Quoc Bao Phan$^{\orcidA{}}$, \textit{Graduate Student Member, IEEE}, and Tuy Tan Nguyen$^{\orcidD{}}$, \textit{Senior Member, IEEE}

\thanks{This research was supported by the Department of Electrical and Computer Engineering, Center for Advanced Power Systems, FAMU-FSU College of Engineering, Florida State University.~{\it (Corresponding author: Tuy Tan Nguyen.)}}
\thanks{The authors are with the Department of Electrical and Computer Engineering, Center for Advanced Power Systems, FAMU-FSU College of Engineering, Florida State University, Tallahassee, FL 32310, USA. (e-mail: qp25c@fsu.edu, tuy.nguyen@fsu.edu).}

\markboth{IEEE Internet of Things Journal}%
{Shell \MakeLowercase{\textit{et al.}}: Bare Demo of IEEEtran.cls for IEEE Journals}}

\maketitle

\begin{abstract}
Federated learning (FL) over mobile and edge devices increasingly involves multimodal models in which clients differ in both sensing capability and computational capacity. Existing update compression schemes typically apply uniform policies across layers and devices, without accounting for modality-specific differences in spectral structure and compressibility. We propose MESH-FL, an entropy-guided matrix product state (MPS) update-compression framework for modality-heterogeneous FL on resource-constrained devices. MESH-FL estimates the spectral entropy of each layer-wise update via truncated singular value decomposition and allocates MPS compression ranks adaptively across layers, modalities, and devices under per-client payload budgets. We show that higher spectral entropy necessitates a higher reconstruction rank under the majorization order on singular-value energy distributions. Building on this result, we prove that the proposed entropy-guided allocation solves a convex surrogate rank-allocation problem, preserves monotonicity under the exact payload model, and achieves convergence with an explicit compression-dependent error term. Experiments on a 15-node heterogeneous Raspberry Pi~4/5 cluster with modality-heterogeneous clients show that MESH-FL achieves up to $56.8\times$ compression while surpassing the uncompressed FedAvg baseline in final accuracy by up to 2.01\%, and reduces total transmitted data to reach convergence by up to $66\times$.
\end{abstract}


\section{Introduction}

Federated learning (FL) has emerged as an important framework for training models across large collections of mobile and edge devices without centralizing raw data \cite{PHAN2026133402, chentmc26}. Its appeal is particularly strong in sensing applications, where data are generated continuously on user devices and are often subject to privacy, bandwidth, and storage constraints \cite{niu2026energy}. At the same time, many edge applications are inherently multimodal. Human activity recognition, mobile health, intelligent environments, and audio-visual perception often rely on multiple sensing streams, while the devices participating in training may not share the same sensing hardware \cite{10938626}. As a result, practical FL deployments increasingly operate under modality heterogeneity, where clients differ not only in computation and communication resources, but also in the modalities they can observe and contribute during training.

This setting gives rise to a communication problem that is more structured than in conventional FL. In cross-device systems, transmitting model updates already accounts for a large fraction of the training cost, especially when clients are resource-constrained and communicate over wireless links \cite{10379499}. In multimodal settings, this challenge becomes more pronounced because different branches of the model can produce updates with different statistical and structural properties. Consequently, a compression policy that is adequate for one modality may be unnecessarily aggressive or insufficient for another. Communication efficiency in multimodal FL is therefore not only a systems issue, but also a representation issue: the compression mechanism should reflect the structure of the layer-wise updates being transmitted.

Existing update compression methods, including sparsification, quantization, and low-rank approximation, have shown that substantial reductions in communication are possible in FL \cite{jia2025communication, 10812017}. However, most existing schemes apply a largely uniform compression policy across layers or across clients, with limited regard for differences in update structure or device capability \cite{10038471}. In heterogeneous edge environments, such uniformity can be inefficient in two ways. First, it can misallocate compression budget by treating updates of different complexity similarly, even though some updates admit stronger compression than others. Second, it can overlook per-device resource constraints, making the same compression choice suboptimal across weak and capable devices. These limitations are particularly relevant in modality-heterogeneous FL, where both model structure and hardware variability shape the communication bottleneck \cite{s23156986}.

Our work is motivated by the observation that the spectral entropy of a layer-wise update matrix, derived from the normalized singular-value energy distribution, provides a compact measure of compressibility: low entropy indicates energy concentrated in a few modes, whereas high entropy reflects a more diffuse spectrum and greater resistance to low-rank approximation. Building on this view, we show that higher spectral entropy generally requires higher reconstruction rank under the natural majorization order on singular-value energy distributions, which motivates entropy-guided rank selection across layers, modalities, and devices. Based on this principle, we develop \textbf{MESH-FL}, an entropy-guided and device-aware compression framework for modality-heterogeneous FL that combines per-layer entropy estimation, adaptive rank allocation under per-client payload budgets, and matrix product state (MPS)-based update compression within a unified pipeline. Instead of enforcing a uniform compression level across all clients and model components, MESH-FL adapts compression to both update complexity and device capability; we further analyze its optimization behavior through an explicit compression-dependent approximation term and evaluate it on a heterogeneous edge testbed and a multimodal benchmark, where it reduces communication while surpassing the uncompressed baseline in accuracy.

Our main contributions are summarized as follows:

\begin{itemize}
    \item We formulate communication-efficient FL under joint modality and device heterogeneity, showing why uniform compression is poorly matched to heterogeneous edge settings.

    \item We introduce a spectral-entropy-guided view of update compression and show that higher spectral entropy implies weakly larger reconstruction rank under the majorization order.

    \item We develop MESH-FL, an entropy-guided MPS compression framework that allocates ranks adaptively across layers, modalities, and clients under per-client payload budgets.

    \item We provide convergence analysis with an explicit compression-dependent error term and validate MESH-FL on heterogeneous Raspberry Pi hardware, demonstrating reduced communication and improved accuracy over baselines.
\end{itemize}

The remainder of this paper is organized as follows. Section~\ref{sec:related} reviews related work and background on update compression and heterogeneous FL. Section~\ref{sec:background} defines the problem formulation. Section~\ref{sec:system} presents the MESH-FL framework and algorithm. Section~\ref{sec:theory} provides the theoretical analysis. Section~\ref{sec:experiments} reports the experimental results. Section~\ref{sec:conclusion} concludes the paper.

\section{Related Work}\label{sec:related}
 
\subsection{Update Compression in Federated Learning}
 
Reducing the per-round communication cost in FL has been pursued through
sparsification, quantization, and low-rank approximation.
Sparsification methods transmit only a subset of update coordinates selected
by magnitude threshold or random sampling~\cite{lin2018dgc,stich2018sparsified}.
With error feedback, these methods converge at rates comparable to dense
communication, but they treat each coordinate independently and cannot exploit
the structural redundancy within weight matrices.
Quantization reduces the bit-width of transmitted values. QSGD~\cite{alistarh2017qsgd}
provides stochastic rounding with variance-bounded convergence guarantees;
TernGrad~\cite{wen2017terngrad} pushes quantization to three levels for
extreme compression. Both are agnostic to the spectral organization of updates
and offer no mechanism to allocate precision proportionally to layer complexity.
 
Low-rank approximation methods are structurally closer to our approach.
PowerSGD~\cite{vogels2019powersgd} compresses gradient matrices via randomized
power iteration and has demonstrated strong empirical compression ratios on
large-scale models. FedPAQ~\cite{reisizadeh2020fedpaq} combines periodic
averaging with quantization. However, existing methods largely apply uniform
rank or budget policies across layers and devices, without adapting to
layer-wise spectral structure or per-client resource constraints. MESH-FL
departs from this uniformity by using spectral entropy to allocate MPS rank
adaptively across layers, modalities, and heterogeneous devices.
 
\subsection{Tensor Decompositions in Neural Networks and Federated Learning}
 
Tensor decompositions have been used extensively for compressing neural network weights~\cite{novikov2015tensorizing,lebedev2015speedingup} and, more recently, for compressing gradients in distributed optimization. The MPS, also known as the tensor-train (TT) decomposition~\cite{oseledets2011tensortraindecomp}, factorizes a reshaped weight or gradient tensor into a chain of three-dimensional cores, enabling compression ratios that scale favorably with tensor order. Authors in~\cite{tjandra2017compressing} applied TT decomposition to recurrent weights; \cite{yang2017tensortrain} extended this to fully connected layers for edge deployment. For a third-order reshaping, the transmitted MPS payload depends directly on the selected bond rank through the core sizes, making rank selection the main mechanism that controls the communication--reconstruction tradeoff. In FL, tensor-based compression has been explored in server-side aggregation and model pruning contexts, but prior work fixes decomposition rank globally and does not adapt to per-layer spectral complexity or per-device payload budgets. MESH-FL combines entropy-guided adaptive MPS rank selection with device-aware budget allocation for modality-heterogeneous FL.
 
\subsection{Heterogeneous Federated Learning}
 
System heterogeneity, where clients differ in computation speed and communication bandwidth, has been addressed through asynchronous aggregation~\cite{xie2019asynchronous}, partial model training~\cite{diao2020heterofl}, and proximal
regularization~\cite{li2020fedprox}. Statistical heterogeneity (non-independent and identically distributed (non-IID) data distributions) has been studied through variance reduction~\cite{karimireddy2020scaffold},
gradient correction, and data-sharing strategies~\cite{zhao2018federated}. These works treat heterogeneity either in the computational or statistical dimension alone and do not address heterogeneous update compression under per-client payload constraints.
 
Modality heterogeneity in FL, where clients differ in their sensing hardware, is less studied. Concurrent work has considered partial-modality participation and modality dropout in cross-silo settings~\cite{sun2024towards}, but without adaptive compression. MESH-FL addresses modality heterogeneity, device heterogeneity, and communication efficiency jointly within a single framework,
providing layer-wise aggregation over modality-eligible clients and entropy-guided rank allocation that respects device-specific payload limits.

\section{Problem Formulation}\label{sec:background}

\subsection{FL Objective}

Consider an FL system with a central server and a set of $K$ clients, indexed by $\mathcal{K}=\{1,2,\ldots,K\}$. Each client $k$ holds a local dataset $\mathcal{D}_k$ of $n_k$ samples, and the datasets are distributed in a non-IID manner across clients. The global learning objectives is:
\begin{equation}
\min_{\mathbf{w}} F(\mathbf{w})
=\sum_{k=1}^{K}\frac{n_k}{n}F_k(\mathbf{w}),
\end{equation}
where \(F_k(\mathbf{w})=\frac{1}{n_k}\sum_{i\in\mathcal{D}_k}\ell(\mathbf{w};x_i)\) is the empirical local objective of client \(k\), \(n=\sum_{k=1}^{K}n_k\) is the total number of samples, \(\ell(\mathbf{w};x_i)\) is the per-sample loss, and \(\mathbf{w}\in\mathbb{R}^d\) denotes the global model parameters. In communication round $t$, the server broadcasts the current global model $\mathbf{w}^t$ to a selected subset $\mathcal{S}^t\subseteq\mathcal{K}$ of participating clients. Each client $k\in\mathcal{S}^t$ performs $E$ steps of local stochastic gradient descent starting from $\mathbf{w}^t$, producing a locally updated model $\mathbf{w}_k^t$. The transmitted local update is the accumulated model delta
\begin{equation}
\Delta\mathbf{w}_k^t=\mathbf{w}^t-\mathbf{w}_k^t,
\end{equation}
which captures the net parameter change after $E$ local steps. The server aggregates these deltas over participating clients only:
\begin{equation}
\mathbf{w}^{t+1}
=\mathbf{w}^t-\sum_{k\in\mathcal{S}^t}\frac{n_k}{n_{\mathcal{S}^t}}\Delta\mathbf{w}_k^t,
\end{equation}
where $n_{\mathcal{S}^t}=\sum_{k\in\mathcal{S}^t}n_k$ is the total number of samples held by participating clients at round $t$. For each layer $\ell$, let $G_k^{(\ell)}\in\mathbb{R}^{m_\ell\times n_\ell}$ denote the layer-wise component of the local update, defined as the difference between the global and locally updated parameters of that layer after $E$ steps. This layer-wise update matrix is compressed before transmission.

\subsection{Modality-Heterogeneous Clients}

We consider a setting in which clients differ in sensing capability. Let $\mathcal{M}$ denote the set of available modalities. Each client $k$ has access to a subset $\mathcal{M}_k\subseteq\mathcal{M}$ determined by its hardware configuration. The global model consists of modality-specific encoder branches and a shared fusion module. Formally,
\begin{equation}
\mathbf{w}=\left\{\mathbf{w}^{(\ell)}\right\}_{\ell=1}^{L},
\end{equation}
where $L$ is the total number of layers and each $\mathbf{w}^{(\ell)}$ belongs either to a modality-specific encoder or to the shared fusion network. Client $k$ computes local updates only for layers associated with its available modalities $\mathcal{M}_k$. Layers corresponding to absent modalities produce no update contribution from client $k$. Let $\mathcal{L}_k$ denote the set of layers updated by client $k$. This naturally partitions clients by modality profile, with unimodal clients contributing updates for a single encoder branch and multimodal clients contributing updates for multiple branches and the fusion module.

\subsection{Device Heterogeneity}

Beyond modality heterogeneity, clients also differ in computation and memory capacity. We characterize each client $k$ by a compression budget $C_k$, defined as the maximum total number of scalar values that may be transmitted per round. This budget serves as a unified abstraction of uplink bandwidth and on-device memory constraints, and is consistently interpreted throughout the paper as a scalar payload limit. We do not assume prior knowledge of the data distribution or update statistics at the server; all compression decisions are made locally at each client based on its own updates and device profile.

\subsection{MPS Update Compression}

Before transmission, each layer-wise update matrix $G_k^{(\ell)}\in\mathbb{R}^{m_\ell\times n_\ell}$ is compressed using a matrix product state (MPS) decomposition. Since MPS operates on higher-order tensors, the update matrix is first reshaped into a third-order tensor $\mathcal{G}_k^{(\ell)}\in\mathbb{R}^{m_1\times m_2\times n_\ell}$, where $m_1$ and $m_2$ are chosen such that $m_1m_2\geq m_\ell$ and $|m_1-m_2|$ is minimized. When $m_1m_2>m_\ell$, the matrix is zero-padded along its first dimension before reshaping, and the padded rows are discarded after reconstruction. This tensorization scheme is the one adopted in our framework; other reshaping strategies are possible but are not considered here. The MPS decomposition expresses this tensor as:
\begin{equation}
\mathcal{G}_k^{(\ell)}(i_1,i_2,i_3)
\approx
\mathbf{A}^{(1)}[i_1]\,
\mathbf{A}^{(2)}[i_2]\,
\mathbf{A}^{(3)}[i_3],
\end{equation}
where $\mathbf{A}^{(1)}\in\mathbb{R}^{m_1\times1\times r}$, $\mathbf{A}^{(2)}\in\mathbb{R}^{r\times m_2\times r}$, and $\mathbf{A}^{(3)}\in\mathbb{R}^{r\times n_\ell\times1}$, with $r=r_k^{(\ell)}$ denoting the bond dimension, or MPS rank. The total payload size of the compressed representation is
\begin{equation}
\phi\!\left(m_\ell,n_\ell,r_k^{(\ell)}\right)=m_1r+m_2r^2+rn_\ell,
\end{equation}
where the dominant term $m_2r^2$ arises from the middle core $\mathbf{A}^{(2)}$, making the payload size quadratic in $r$. The corresponding compression ratio is
\begin{equation}
\rho_k^{(\ell)}
=
\frac{m_\ell n_\ell}{\phi(m_\ell,n_\ell,r_k^{(\ell)})}.
\end{equation}
Thus, $r_k^{(\ell)}$ is the primary decision variable governing the trade-off between payload size and reconstruction fidelity.

\subsection{Problem Statement}

The central challenge is to select compression ranks $\{r_k^{(\ell)}\}$ across clients and layers so as to minimize total update reconstruction error subject to per-client compression budgets:
\begin{align}
\min_{\left\{r_k^{(\ell)}\right\}}
\sum_{k\in\mathcal{S}^t}\frac{n_k}{n_{\mathcal{S}^t}}
\sum_{\ell\in\mathcal{L}_k}\varepsilon_{k,\ell}^2\!\left(r_k^{(\ell)}\right)
\quad \notag \\
\text{s.t.}
\quad
\sum_{\ell\in\mathcal{L}_k}\phi\!\left(m_\ell,n_\ell,r_k^{(\ell)}\right)\leq C_k,
\quad
\forall k\in\mathcal{S}^t,
\end{align}
where
$\varepsilon_{k,\ell}(r)
=
\frac{\|G_k^{(\ell)}-\hat{G}_k^{(\ell)}\|_F}{\|G_k^{(\ell)}\|_F}$ is the relative reconstruction error of layer $\ell$ at client $k$ under rank-$r$ MPS compression, and $\hat{G}_k^{(\ell)}$ denotes the reconstructed update. The constraint is dimensionally consistent: the left-hand side counts the total number of transmitted scalars under MPS compression, and $C_k$ is measured in the same units. This formulation captures both sources of heterogeneity jointly: the constraint set is client-specific through $C_k$, while the objective reflects modality participation through $\mathcal{L}_k$.

Solving this problem requires understanding how $\varepsilon_{k,\ell}(r)$ depends on the intrinsic structure of the update $G_k^{(\ell)}$. The key insight developed in Section~\ref{sec:theory} is that the spectral entropy of $G_k^{(\ell)}$ provides a principled basis for rank selection. In our framework, spectral entropy is estimated efficiently via a truncated singular value decomposition that retains only the top 
$q$ singular values to estimate the singular-value energy distribution, where $q\ll\min(m_\ell,n_\ell)$, thereby avoiding the cost of full decomposition at runtime, as detailed in Section~\ref{sec:system}.

\section{MESH-FL Framework}\label{sec:system}

MESH-FL operates in three steps per round: each participating client estimates the spectral entropy of its layer-wise updates via truncated SVD, allocates entropy-guided MPS ranks under its local payload budget $C_k$, and compresses active updates into MPS cores for transmission. The server reconstructs the received updates, aggregates them layer-wise across eligible clients, and broadcasts the updated global model. Only global parameters in modality-specific encoders and the shared fusion module are transmitted, while local classifier and fusion-adapter parameters remain on-device.

\subsection{Spectral Entropy Estimation}
For each active layer $\ell\in\mathcal{L}_k$, client $k$ estimates the spectral entropy of its update matrix $G_k^{(\ell)}\in\mathbb{R}^{m_\ell\times n_\ell}$ via truncated SVD.
Let
\begin{align}
\hat{\sigma}_1\geq \hat{\sigma}_2\geq \cdots \geq \hat{\sigma}_q
\end{align}
denote the top $q$ singular values of $G_k^{(\ell)}$, where $q\ll\min(m_\ell,n_\ell)$ is a fixed truncation parameter. The normalized singular-value energy distribution over the retained components is
\begin{align}
\hat{q}_i=\frac{\hat{\sigma}_i^2}{\sum_{j=1}^{q}\hat{\sigma}_j^2},
\quad i=1,\ldots,q,
\end{align}
and the resulting truncated spectral entropy estimator is
\begin{align}
\hat{H}_k^{(\ell)}=-\sum_{i=1}^{q}\hat{q}_i\log \hat{q}_i.
\end{align}

We use $\hat{H}_k^{(\ell)}$ as a computational proxy for the true spectral entropy $H_k^{(\ell)}$ defined over the full singular-value energy distribution. The quality of this proxy depends on the omitted tail energy
\begin{align}
\tau_k^{(\ell)}=
\frac{\sum_{i>q}\sigma_i^2}{\|G_k^{(\ell)}\|_F^2},
\end{align}
which measures the fraction of spectral energy not captured by the top $q$ singular values. When $\tau_k^{(\ell)}$ is small, the truncated entropy $\hat{H}_k^{(\ell)}$ approximates $H_k^{(\ell)}$ up to a controlled deviation. In particular, the entropy-guided rank ordering induced by $\hat{H}$ is stable under small tail energy, and pairwise ordering is preserved whenever the true entropy gaps exceed the corresponding perturbation bound. 

In practice, $q$ is chosen so that the retained singular values capture the dominant spectral mass of each update and the resulting entropy estimates are stable across rounds. Using truncated rather than full SVD reduces the per-layer cost from $O(m_\ell n_\ell \min(m_\ell,n_\ell))$ to $O(m_\ell n_\ell q)$ with randomized SVD, which is practical on resource-constrained devices. Entropy estimation is performed independently for each active layer, allowing rank allocation to adapt to layer-specific and modality-specific update structure.

\subsection{Device-Aware Rank Allocation}

Given the entropy estimates $\{\hat{H}_k^{(\ell)}\}_{\ell\in\mathcal{L}_k}$ and payload budget $C_k$, client $k$ allocates MPS ranks across its active layers. As established in Section~\ref{sec:theory}, the nominal continuous allocation is monotone in spectral entropy, and under the linearized payload model admits the closed-form scaling $r_k^{(\ell)}\propto e^{\hat{H}_k^{(\ell)}/2}$. We therefore define
\begin{align}
\bar{r}_k^{(\ell)}=\alpha_k e^{\hat{H}_k^{(\ell)}/2},
\quad \ell\in\mathcal{L}_k,
\end{align}
where $\alpha_k>0$ is a client-specific scaling factor chosen so that the resulting allocation matches the available payload budget as closely as possible under the MPS payload model. Since the payload function
\begin{align}
\phi(m_\ell,n_\ell,r)=m_1r+m_2r^2+rn_\ell
\end{align}
is monotone increasing in $r$, the scaling factor $\alpha_k$ can be obtained efficiently by bisection on
\begin{align}
\sum_{\ell\in\mathcal{L}_k}\phi\!\left(m_\ell,n_\ell,\bar{r}_k^{(\ell)}\right)\leq C_k.
\end{align}

To obtain implementable integer ranks, the continuous allocation is projected onto device-feasible bounds:
\begin{align}
\tilde{r}_k^{(\ell)}
=
\Pi_{\mathbb{Z}\cap[r_{\min},\,r_{\max}^{(k)}]}
\!\left(\bar{r}_k^{(\ell)}\right),
\end{align}
where $r_{\min}$ is a global minimum rank and $r_{\max}^{(k)}$ is a device-specific maximum rank determined by client $k$'s available working memory. Because integer projection may slightly violate the payload budget, we apply a lightweight feasibility correction: if
\begin{align}
\sum_{\ell\in\mathcal{L}_k}\phi\!\left(m_\ell,n_\ell,\tilde{r}_k^{(\ell)}\right)>C_k,
\end{align}
the client iteratively decreases the ranks of the lowest-entropy layers until the budget is satisfied; if residual budget remains, it is reassigned to the highest-entropy layers while respecting $r_{\max}^{(k)}$. The final ranks are denoted by $\{r_k^{(\ell)}\}_{\ell\in\mathcal{L}_k}$.

This procedure preserves the entropy-guided structure implied by the theory while ensuring exact budget feasibility under integer ranks and device bounds. Layers with higher spectral entropy receive higher rank, layers with lower entropy are compressed more aggressively, and every client remains within its payload limit. For bias vectors and other one-dimensional parameter tensors, we directly assign rank $r_{\min}$, since their small size makes both entropy estimation and compression gains marginal.

\subsection{MPS Compression and Decompression}

Given the allocated rank $r_k^{(\ell)}$ for each active layer, client $k$ compresses the update matrix $G_k^{(\ell)}$ using the MPS construction introduced in Section~\ref{sec:background}. If necessary, the matrix is zero-padded, reshaped into a third-order tensor, and factorized through two sequential truncated SVD steps to produce the cores $\mathbf{A}^{(1)}$, $\mathbf{A}^{(2)}$, and $\mathbf{A}^{(3)}$. The singular values from each intermediate decomposition are absorbed into the subsequent core so that the server can reconstruct the update directly from the transmitted cores without storing singular values separately. The total transmitted payload for layer $\ell$ at client $k$ is therefore
\begin{equation}
\phi\!\left(m_\ell,n_\ell,r_k^{(\ell)}\right)=m_1r_k^{(\ell)}+m_2\bigl(r_k^{(\ell)}\bigr)^2+r_k^{(\ell)}n_\ell.
\end{equation}

At the server, decompression contracts the received cores sequentially to recover $\hat{G}_k^{(\ell)}$, removes any padded rows, and reshapes the result to the original layer dimensions. Since decompression is performed on the server side, its overhead is modest relative to server-scale compute resources. The reconstructed layer-wise updates are then aggregated across clients as in FL, with the important distinction that modality-specific layers are aggregated only over clients that actually contribute to them.

\subsection{Layer-Wise Aggregation Under Modality Heterogeneity}

Because not every client updates every layer, aggregation must be defined layer-wise. For each layer $\ell$, let $\mathcal{S}_\ell^t=\{k\in\mathcal{S}^t:\ell\in\mathcal{L}_k\}$ denote the set of participating clients that contribute an update to layer $\ell$, and let $n_{\mathcal{S}_\ell^t}=\sum_{k\in\mathcal{S}_\ell^t} n_k$ denote the total number of samples represented by those clients. The server then computes the aggregated update for layer $\ell$ as
\begin{align}
\hat{G}^{(\ell)}=
\sum_{k\in\mathcal{S}_\ell^t}
\frac{n_k}{n_{\mathcal{S}_\ell^t}}
\hat{G}_k^{(\ell)}.
\end{align}
If $\mathcal{S}_\ell^t=\emptyset$, no client contributes to layer $\ell$ in round $t$, and the corresponding parameters remain unchanged. This layer-wise normalization avoids diluting modality-specific updates by averaging over clients that do not possess the relevant sensing branch.

\subsection{Complete Protocol}
Algorithm~\ref{alg:mesh-fl} summarizes the complete MESH-FL workflow. In each communication round, the server broadcasts the current global model to the selected clients, and each client performs $E$ local SGD steps to obtain layer-wise updates only for the layers in its modality-dependent active set $\mathcal{L}_k$. For each active layer, the client estimates the truncated spectral entropy $\hat{H}_k^{(\ell)}$ from the top-$q$ singular-value energy distribution and uses these estimates to allocate MPS ranks $\{r_k^{(\ell)}\}_{\ell\in\mathcal{L}_k}$ under the payload budget $\sum_{\ell\in\mathcal{L}_k}\phi(m_\ell,n_\ell,r_k^{(\ell)})\le C_k$. The resulting ranks assign higher reconstruction capacity to spectrally diffuse updates and stronger compression to low-entropy updates. Clients then transmit only the compressed MPS cores, while the server reconstructs the updates and aggregates each layer over the modality-eligible clients $S_\ell^t={k\in S^t:\ell\in\mathcal{L}_k}$. Thus, compared with standard FedAvg, MESH-FL jointly introduces entropy-guided rank allocation, device-aware MPS update compression, and layer-wise aggregation under modality heterogeneity.

\begin{algorithm}[!t]
\DontPrintSemicolon
\SetNoFillComment
\caption{MESH-FL Work Flow}
\label{alg:mesh-fl}
\KwIn{$\mathbf{w}^0$, rounds $T$, local steps $E$, truncation $q$, rank bounds $r_{\min},\{r_{\max}^{(k)}\}$, budgets $\{C_k\}$}
\KwOut{$\mathbf{w}^T$}

\For{$t=1,\ldots,T$}{
    Server selects $\mathcal{S}^t\subseteq\mathcal{K}$ and broadcasts $\mathbf{w}^t$\;

    \ForEach{$k\in\mathcal{S}^t$ in parallel}{
        Run $E$ local SGD steps from $\mathbf{w}^t$ to obtain $\mathbf{w}_k^t$\;

        \ForEach{$\ell\in\mathcal{L}_k$}{
            $G_k^{(\ell)}\leftarrow \mathbf{w}^{t,(\ell)}-\mathbf{w}_k^{t,(\ell)}$\;
            Estimate $\hat{H}_k^{(\ell)}$ from the top-$q$ singular-value energy of $G_k^{(\ell)}$\;
        }

        Allocate $\{r_k^{(\ell)}\}_{\ell\in\mathcal{L}_k}$ using $\{\hat{H}_k^{(\ell)}\}$ subject to
        $\sum_{\ell\in\mathcal{L}_k}\phi(m_\ell,n_\ell,r_k^{(\ell)})\leq C_k$\;

        Compress and transmit MPS cores of $G_k^{(\ell)}$ at rank $r_k^{(\ell)}$ for all $\ell\in\mathcal{L}_k$\;
        Update $\mathbf{w}_k^{\mathrm{loc}}$\;
    }

    \ForEach{$\ell=1,\ldots,L$}{
        $\mathcal{S}_\ell^t\leftarrow\{k\in\mathcal{S}^t:\ell\in\mathcal{L}_k\}$\;
        \eIf{$\mathcal{S}_\ell^t\neq\emptyset$}{
            $\hat{G}^{(\ell)}\leftarrow
            \sum_{k\in\mathcal{S}_\ell^t}
            \frac{n_k}{n_{\mathcal{S}_\ell^t}}
            \operatorname{MPSDec}\!\left(\operatorname{cores}_k^{(\ell)}\right)$\;
            $\mathbf{w}^{t+1,(\ell)}\leftarrow \mathbf{w}^{t,(\ell)}-\hat{G}^{(\ell)}$\;
        }{
            $\mathbf{w}^{t+1,(\ell)}\leftarrow \mathbf{w}^{t,(\ell)}$\;
        }
    }
}
\KwRet{$\mathbf{w}^T$}
\end{algorithm}

\subsection{Complexity Analysis}

Relative to standard FedAvg, MESH-FL introduces two main client-side overheads: entropy estimation and MPS compression. For client $k$, truncated-SVD-based entropy estimation has complexity
\begin{equation}
\sum_{\ell\in\mathcal{L}_k}O(m_\ell n_\ell q),
\end{equation}
where $q\ll\min(m_\ell,n_\ell)$. The MPS compression cost is approximately
\begin{equation}
\sum_{\ell\in\mathcal{L}_k}O(m_\ell n_\ell r_k^{(\ell)}),
\end{equation}
which is comparable to low-rank factorization at the target rank. The rank-allocation step is lightweight relative to these matrix operations, since it consists of entropy-weight computation, one-dimensional bisection for $\alpha_k$, and a small feasibility correction over active layers.

On the server side, decompression and aggregation incur
\begin{equation}
\sum_{\ell=1}^{L}\sum_{k\in\mathcal{S}_\ell^t}
O\!\left(m_1 m_2 \bigl(r_k^{(\ell)}\bigr)^2
+ m_1 m_2\, r_k^{(\ell)} n_\ell\right),
\end{equation}
which is modest at server scale. The main efficiency gain remains in communication: the payload per client per round is reduced from
\begin{equation}
\sum_{\ell\in\mathcal{L}_k}m_\ell n_\ell
\end{equation}
to
\begin{equation}
\sum_{\ell\in\mathcal{L}_k}\phi(m_\ell,n_\ell,r_k^{(\ell)}),
\end{equation}
with larger savings obtained when the allocated ranks remain small relative to the layer dimensions.

\section{Theoretical Analysis}\label{sec:theory}
This section provides the theoretical basis for MESH-FL. We first bound the reconstruction error of rank-constrained MPS compression and express it through an entropy-weighted spectral-complexity coefficient. We then relate spectral entropy to rank complexity, show how compression error enters the FL convergence bound, and prove that the proposed entropy-guided rank allocation solves a convex surrogate problem.

\subsection{Assumptions}\label{sec:assumptions}
We use standard conditions from non-IID FL convergence analysis. Each local objective \(F_k\) is differentiable and \(L\)-smooth, i.e.,
\begin{align}
\|\nabla F_k(\mathbf{u})-\nabla F_k(\mathbf{v})\|\leq L\|\mathbf{u}-\mathbf{v}\|,\quad \forall\,\mathbf{u},\mathbf{v}\in\mathbb{R}^d.
\label{eq:smoothness}
\end{align}
The uncompressed local delta \(G_k^t\), produced after \(E\) local SGD steps from \(\mathbf{w}^t\), satisfies
\begin{align}
\mathbb{E}[G_k^t\mid\mathbf{w}^t]&=\eta E\nabla F_k(\mathbf{w}^t),\label{eq:local-delta-mean}\\
\mathbb{E}\!\left[\|G_k^t-\eta E\nabla F_k(\mathbf{w}^t)\|^2\mid\mathbf{w}^t\right]&\leq\eta^2E^2\sigma^2,
\label{eq:local-delta-var}
\end{align}
where \(\eta\) is the step size and \(\sigma^2\) controls the stochastic local-update variance. Client heterogeneity is bounded by
\begin{align}
\|\nabla F_k(\mathbf{w})-\nabla F(\mathbf{w})\|^2\leq\delta^2,\quad
\nabla F(\mathbf{w})=\sum_{k=1}^K\frac{n_k}{n}\nabla F_k(\mathbf{w}).
\label{eq:gradient-dissimilarity}
\end{align}
These conditions are standard and are used only for the convergence analysis. The following assumption is specific to the spectral approximation analysis of MESH-FL.

\begin{assumption}[Exponential tail decay of unfolding spectra]\label{assump:exp-tail}
For each unfolding \(s\in\{1,2\}\), there exist constants \(A_s>0\) and \(\alpha_s>0\) such that the normalized tail energy satisfies
\begin{align}
\sum_{j>r}q_j^{(s)}\leq A_s e^{-\alpha_s r},\quad r=1,2,\ldots.
\label{eq:exp-tail}
\end{align}
\end{assumption}
Assumption~\ref{assump:exp-tail} is a sufficient technical condition for deriving a closed-form entropy-weighted MPS approximation bound. It is used only to establish Proposition~\ref{prop:entropy-mps} and is not required for the convergence argument in Theorem~\ref{thm:convergence} or for the convexity and monotonicity of the rank-allocation problem in Theorem~\ref{thm:allocation}. The assumption is not meant to require that every layer and every communication round exhibit exact exponential spectral decay. In practice, the allocation rule relies on the weaker ordering principle that layers with more diffuse singular-value energy distributions require larger reconstruction ranks. Thus, even when empirical tail decay is slower than exponential, the \(1/r\)-type surrogate remains a tractable proxy for rank allocation, while the constants and tightness of the approximation bound may become less favorable.

\subsection{MPS Approximation Error and Entropy Complexity}\label{sec:mps-error}
We first quantify the approximation error introduced by sequential MPS compression. For each unfolding \(s\in\{1,2\}\), define the normalized singular-value energy distribution and its Shannon entropy as
\begin{align}
q_i^{(s)}=\frac{(\sigma_i^{(s)})^2}{\|\mathcal{G}^{(s)}\|_F^2},\quad
H^{(s)}(G)=-\sum_{i\geq1}q_i^{(s)}\log q_i^{(s)}.
\label{eq:unfolding-entropy}
\end{align}
The following proposition converts the discarded spectral energy of the two MPS unfoldings into a rank-dependent reconstruction-error bound.

\begin{proposition}[Entropy-weighted MPS error bound]\label{prop:entropy-mps}
Let \(G\in\mathbb{R}^{m\times n}\) be reshaped into a third-order tensor \(\mathcal{G}\in\mathbb{R}^{m_1\times m_2\times n}\), with first unfolding \(\mathcal{G}^{(1)}\in\mathbb{R}^{m_1\times m_2n}\) and second unfolding \(\mathcal{G}^{(2)}\) obtained after the first rank-\(r\) truncation. Under Assumption~\ref{assump:exp-tail}, the relative MPS reconstruction error of layer \(\ell\) at client \(k\) satisfies
\begin{align}
\varepsilon_{k,\ell}^2(r)=\frac{\|G_k^{(\ell)}-\widehat{G}_{k,r}^{(\ell)}\|_F^2}{\|G_k^{(\ell)}\|_F^2}\leq\frac{\Psi_{k,\ell}}{r},
\label{eq:error-bound}
\end{align}
where $\Psi_{k,\ell}:=C_{1,\ell}e^{H_{k,\ell}^{(1)}}+C_{2,\ell}e^{H_{k,\ell}^{(2)}},$ \(H_{k,\ell}^{(s)}:=H^{(s)}(G_k^{(\ell)})\) is the energy entropy of the \(s\)th unfolding, and \(C_{1,\ell},C_{2,\ell}>0\) depend only on the tail parameters \((A_s,\alpha_s)\) and the fixed tensorization of layer \(\ell\). 
\end{proposition}

\begin{proof}
The sequential MPS construction applies two rank-\(r\) truncated SVD steps. The first truncation is performed on the unfolding \(\mathcal{G}^{(1)}\), and by the Eckart--Young theorem it introduces an error \(\mathbf{e}_1\) satisfying
\begin{align}
\|\mathbf{e}_1\|_F^2=\sum_{j>r}\bigl(\sigma_j^{(1)}\bigr)^2.
\label{eq:app-first-svd-error}
\end{align}
After reshaping the retained factor from the first step, the second truncated SVD introduces an error \(\mathbf{e}_2\) satisfying
\begin{align}
\|\mathbf{e}_2\|_F^2\leq\sum_{j>r}\bigl(\sigma_j^{(2)}\bigr)^2.
\label{eq:app-second-svd-error}
\end{align}
Since the final reconstruction error can be written as
\begin{align}
\|G - \widehat{G}_r\|_F \leq \|e_1\|_F + \|e_2\|_F,
\label{eq:app-total-mps-error}
\end{align}
we have
\begin{align}
\|G-\widehat{G}_r\|_F^2
&=\|\mathbf{e}_1+\mathbf{e}_2\|_F^2\notag\\
&\leq2\|\mathbf{e}_1\|_F^2+2\|\mathbf{e}_2\|_F^2\notag\\
&\leq2\sum_{j>r}\bigl(\sigma_j^{(1)}\bigr)^2+2\sum_{j>r}\bigl(\sigma_j^{(2)}\bigr)^2.
\label{eq:app-sequential-mps-bound}
\end{align}
For each unfolding \(s\in\{1,2\}\), Assumption~\ref{assump:exp-tail} gives
\begin{align}
\sum_{j>r}\bigl(\sigma_j^{(s)}\bigr)^2
&=\|\mathcal{G}^{(s)}\|_F^2\sum_{j>r}q_j^{(s)}\notag\\
&\leq A_s\|\mathcal{G}^{(s)}\|_F^2e^{-\alpha_s r}.
\label{eq:app-tail-unfolding}
\end{align}
For any \(\alpha>0\) and \(r\geq1\), \(e^{-\alpha r}\leq(e\alpha r)^{-1}\), because \(xe^{-x}\leq e^{-1}\) for all \(x>0\). Therefore,
\begin{align}
\sum_{j>r}\bigl(\sigma_j^{(s)}\bigr)^2
\leq\frac{A_s}{e\alpha_s}\frac{\|\mathcal{G}^{(s)}\|_F^2}{r}.
\label{eq:app-tail-r-bound}
\end{align}
Let \(R_{\mathrm{eff}}^{(s)}(G):=e^{H^{(s)}(G)}\). Since \(R_{\mathrm{eff}}^{(s)}(G)\geq1\), the bound can be relaxed into the entropy-weighted form
\begin{align}
\sum_{j>r}\bigl(\sigma_j^{(s)}\bigr)^2
\leq\frac{\widetilde{C}_{s,\ell}e^{H_{k,\ell}^{(s)}}}{r}\|\mathcal{G}^{(s)}\|_F^2,
\label{eq:app-tail-entropy-bound}
\end{align}
where \(\widetilde{C}_{s,\ell}>0\) absorbs the tail-decay constants and the fixed tensorization of layer \(\ell\). Applying \eqref{eq:app-sequential-mps-bound} to layer \(\ell\) of client \(k\) gives
\begin{align}
\|G_k^{(\ell)}-\widehat{G}_{k,r}^{(\ell)}\|_F^2
&\leq2\sum_{j>r}\bigl(\sigma_j^{(1)}\bigr)^2+2\sum_{j>r}\bigl(\sigma_j^{(2)}\bigr)^2\notag\\
&\leq\frac{C_{1,\ell}e^{H_{k,\ell}^{(1)}}}{r}\|\mathcal{G}^{(1)}\|_F^2+\frac{C_{2,\ell}e^{H_{k,\ell}^{(2)}}}{r}\|\mathcal{G}^{(2)}\|_F^2,
\label{eq:app-mps-before-normalization}
\end{align}
where the constants \(C_{1,\ell}\) and \(C_{2,\ell}\) absorb the factor \(2\). Since \(\|\mathcal{G}^{(1)}\|_F=\|G_k^{(\ell)}\|_F\) and \(\|\mathcal{G}^{(2)}\|_F\leq\|G_k^{(\ell)}\|_F\), dividing both sides by \(\|G_k^{(\ell)}\|_F^2\) yields
\begin{align}
\varepsilon_{k,\ell}^2(r)
&=\frac{\|G_k^{(\ell)}-\widehat{G}_{k,r}^{(\ell)}\|_F^2}{\|G_k^{(\ell)}\|_F^2}\notag\\
&\leq\frac{C_{1,\ell}e^{H_{k,\ell}^{(1)}}+C_{2,\ell}e^{H_{k,\ell}^{(2)}}}{r}
=\frac{\Psi_{k,\ell}}{r}.
\label{eq:app-final-entropy-mps-bound}
\end{align}
This proves the stated bound.
\end{proof}

The coefficient \(\Psi_{k,\ell}\) provides a conservative entropy-weighted measure of layer-wise spectral complexity. Larger \(\Psi_{k,\ell}\) corresponds to a larger upper bound on the compression error at the same rank, motivating larger rank allocation for spectrally more complex updates. This leads to the surrogate allocation problem
\begin{align}
\min_{\{r_k^{(\ell)}>0\}}\sum_{\ell\in\mathcal{L}_k}\frac{\Psi_{k,\ell}}{r_k^{(\ell)}}\quad\mathrm{s.t.}\quad\sum_{\ell\in\mathcal{L}_k}\phi_\ell\!\left(r_k^{(\ell)}\right)\leq C_k.
\label{eq:psi-surrogate}
\end{align}
This surrogate is optimized in Section~\ref{sec:optimality}. In the implemented algorithm, the truncated layer-wise entropy \(\widehat{H}_k^{(\ell)}\) serves as a scalar proxy for the spectral-complexity weight \(\Psi_{k,\ell}\). Although \(\Psi_{k,\ell}\) is defined through the two sequential MPS unfoldings, both quantities reflect spectral diffuseness: \(\Psi_{k,\ell}\) increases when either unfolding entropy is large, whereas \(\widehat{H}_k^{(\ell)}\) increases when the singular-value energy of the original matrix is diffuse. Since reshaping preserves the Frobenius norm but does not generally preserve the singular spectrum, no exact equivalence between \(\widehat{H}_k^{(\ell)}\) and \(H^{(1)}_{k,\ell}\) is assumed. Instead, \(\widehat{H}_k^{(\ell)}\) is used as a low-cost proxy for the overall layer-wise spectral complexity. This approximation is most appropriate when the first truncation captures the dominant spectral mass, thereby reducing the practical influence of the second truncation error, consistent with the compressible-spectrum regime described by Assumption~\ref{assump:exp-tail}.

\subsection{Spectral Entropy and Rank Complexity}\label{sec:entropy-rank}
The preceding bound uses entropy to measure the spectral complexity of compressed updates. We now define the corresponding matrix-level entropy and the minimum rank needed to achieve a target reconstruction error.

\begin{definition}[Energy entropy and effective rank]\label{def:energy-entropy}
Let \(G\in\mathbb{R}^{m\times n}\) have singular values \(\sigma_1\geq\cdots\geq\sigma_p>0\), where \(p=\min(m,n)\). Define
\begin{align}
q_i=\frac{\sigma_i^2}{\|G\|_F^2},\quad
H(G)=-\sum_{i=1}^{p}q_i\log q_i,\quad
r_{\mathrm{eff}}(G)=e^{H(G)}.
\label{eq:energy-entropy}
\end{align}
For \(\varepsilon\in(0,1)\), define the minimum \(\varepsilon\)-accurate rank as
\begin{align}
r^*(\varepsilon;G)=\min\left\{r:\sum_{i>r}q_i\leq\varepsilon^2\right\}.
\label{eq:epsilon-rank}
\end{align}
\end{definition}
MESH-FL does not compute the full spectrum during training. Each client estimates entropy from the leading singular values only. The next remark states that when this truncated estimator remains close to the full entropy and preserves entropy ordering.

\begin{remark}[Truncated entropy stability]\label{rem:truncated-entropy}
Let \(\tau=\sum_{i>q}q_i\) denote the omitted tail energy and let \(\widehat{H}(G)\) be the entropy computed from the renormalized top-\(q\) energy weights \(\widehat{q}_i=q_i/(1-\tau)\). Then
\begin{align}
H(G)=h_2(\tau)+(1-\tau)\widehat{H}(G)+\tau H_{\mathrm{tail}}(G),
\label{eq:entropy-decomposition}
\end{align}
where \(h_2(\tau)=-\tau\log\tau-(1-\tau)\log(1-\tau)\). Moreover,
\begin{align}
|H(G)-\widehat{H}(G)|\leq h_2(\tau)+\tau\log p.
\label{eq:entropy-approx}
\end{align}
Thus, \(\widehat{H}(G)\to H(G)\) as \(\tau\to0\). For two matrices \(G\) and \(G'\), with perturbation margins \(\Delta(G)=h_2(\tau)+\tau\log p\) and \(\Delta(G')=h_2(\tau')+\tau'\log p'\), the ordering \(H(G)>H(G')\) is guaranteed whenever
\begin{align}
\widehat{H}(G)>\widehat{H}(G')+\Delta(G)+\Delta(G').
\label{eq:entropy-ordering-stability}
\end{align}
\end{remark}
The following theorem provides the main ordering principle behind entropy-guided compression. Under the natural majorization order, a more diffuse singular-value energy distribution has larger entropy and requires a weakly larger rank to achieve the same approximation target.

\begin{theorem}[Entropy-rank monotonicity]\label{thm:entropy-rank}
Let \(q=(q_1,\ldots,q_p)\) and \(q'=(q_1',\ldots,q_p')\) be two energy distributions sorted in nonincreasing order. If \(q\) majorizes \(q'\), then
\begin{align}
H(q)\leq H(q'),\quad
r^*(\varepsilon;q)\leq r^*(\varepsilon;q'),\quad\forall\,\varepsilon\in(0,1).
\label{eq:entropy-rank-monotonicity}
\end{align}
\end{theorem}
\begin{proof}
Shannon entropy is Schur-concave, so \(q\succ q'\) implies \(H(q)\leq H(q')\). Majorization also gives
\begin{align}
\sum_{i=1}^r q_i\geq\sum_{i=1}^r q_i',\quad\forall\,r,
\label{eq:majorization-prefix}
\end{align}
which is equivalent to
\begin{align}
\sum_{i>r}q_i\leq\sum_{i>r}q_i',\quad\forall\,r.
\label{eq:majorization-tail}
\end{align}
Therefore, the smallest rank achieving tail energy at most \(\varepsilon^2\) satisfies \(r^*(\varepsilon;q)\leq r^*(\varepsilon;q')\).
\end{proof}

Proposition~\ref{prop:entropy-mps}, Remark~\ref{rem:truncated-entropy}, and Theorem~\ref{thm:entropy-rank} together provide the spectral foundation of MESH-FL. Proposition~\ref{prop:entropy-mps} links MPS rank to reconstruction error through an entropy-weighted complexity coefficient. Remark~\ref{rem:truncated-entropy} justifies the use of truncated entropy when the omitted tail energy is small. Theorem~\ref{thm:entropy-rank} shows that entropy ordering is consistent with rank-complexity ordering under majorization.

\subsection{Convergence of MESH-FL}\label{sec:convergence}
The previous results characterize the approximation error introduced by MPS compression. We now show how this error enters the optimization dynamics of MESH-FL.

\begin{theorem}[Convergence with compression error]\label{thm:convergence}
Under the conditions in Section~\ref{sec:assumptions}, assume full participation and \(\eta\leq1/(4LE)\). Let \(\widehat{G}_k^t=G_k^t+\mathbf{e}_k^t\) be the reconstructed update received from client \(k\) at round \(t\), where \(\mathbf{e}_k^t\) is the compression error. Then
\begin{align}
&\frac{1}{T}\sum_{t=0}^{T-1}\mathbb{E}\|\nabla F(\mathbf{w}^t)\|^2
\leq\frac{4(F(\mathbf{w}^0)-F^*)}{\eta ET}+4L\eta E\sigma^2\notag\\&+8L\eta E\delta^2
+\frac{5}{\eta^2E^2T}\sum_{t=0}^{T-1}\sum_{k=1}^{K}\frac{n_k}{n}\mathbb{E}\|\mathbf{e}_k^t\|^2.
\label{eq:convergence-bound}
\end{align}
If
\begin{align}
\|\mathbf{e}_k^t\|^2\leq2\sum_{\ell\in\mathcal{L}_k}\varepsilon_{k,\ell}^2\!\left(r_k^{(\ell)}\right)\|G_k^{(\ell),t}\|_F^2,
\label{eq:compression-error-layerwise}
\end{align}
then the final term in \eqref{eq:convergence-bound} is explicitly controlled by the layer-wise MPS reconstruction errors. 
\end{theorem}

\begin{proof}
Let $p_k = n_k/n$, $\overline{G}^{t} = \sum_k p_k G_k^t$, and
$\overline{\mathbf{e}}^{t} = \sum_k p_k \mathbf{e}_k^t$, so the
MESH-FL update satisfies
$\mathbf{w}^{t+1} - \mathbf{w}^{t} = -\overline{G}^{t} - \overline{\mathbf{e}}^{t}$.
By $L$-smoothness and the local-delta mean condition
\eqref{eq:local-delta-mean},
\begin{align}
\mathbb{E}\!\left[F(\mathbf{w}^{t+1})\mid\mathbf{w}^{t}\right]
&\leq F(\mathbf{w}^{t})
  - \eta E\|\nabla F(\mathbf{w}^{t})\|^2 \notag \\
 & - \left\langle\nabla F(\mathbf{w}^{t}),\overline{\mathbf{e}}^{t}\right\rangle
  + \frac{L}{2}\|\overline{G}^{t}+\overline{\mathbf{e}}^{t}\|^2.
\label{eq:app-smooth-main}
\end{align}
Applying Young's inequality to the inner product and Jensen's
inequality to $\|\overline{\mathbf{e}}^{t}\|^2$,
\begin{align}
-\left\langle\nabla F(\mathbf{w}^{t}),\overline{\mathbf{e}}^{t}\right\rangle
\leq \frac{\eta E}{4}\|\nabla F(\mathbf{w}^{t})\|^2
  + \frac{1}{\eta E}\sum_{k}p_k\|\mathbf{e}_k^t\|^2.
\label{eq:app-young-jensen}
\end{align}
For the quadratic term, writing
$G_k^t = \eta E\nabla F_k(\mathbf{w}^{t}) + \boldsymbol{\xi}_k^t$
with $\mathbb{E}[\boldsymbol{\xi}_k^t]=0$ and
$\mathbb{E}[\|\boldsymbol{\xi}_k^t\|^2]\leq\eta^2E^2\sigma^2$,
and using the gradient-dissimilarity bound \eqref{eq:gradient-dissimilarity},
\begin{align}
&\mathbb{E}\!\left[\|\overline{G}^{t}\|^2\mid\mathbf{w}^{t}\right]
\leq \sum_k p_k\mathbb{E}\!\left[\|G_k^t\|^2\mid\mathbf{w}^{t}\right] \notag \\
&\leq 2\eta^2E^2\|\nabla F(\mathbf{w}^{t})\|^2
  + 2\eta^2E^2\delta^2 + \eta^2E^2\sigma^2.
\label{eq:app-gbar-bound}
\end{align}
Substituting \eqref{eq:app-young-jensen}--\eqref{eq:app-gbar-bound}
into \eqref{eq:app-smooth-main} and applying $\eta \leq 1/(4LE)$
(so that $2L\eta^2E^2 \leq \eta E/2$ and $L \leq 1/(4\eta E)$) gives
\begin{align}
&\mathbb{E}\!\left[F(\mathbf{w}^{t+1})\mid\mathbf{w}^{t}\right]
\leq F(\mathbf{w}^{t})
  - \frac{\eta E}{4}\|\nabla F(\mathbf{w}^{t})\|^2 \notag \\
&+ L\eta^2E^2\sigma^2
  + 2L\eta^2E^2\delta^2
  + \frac{5}{4\eta E}\sum_{k}p_k\mathbb{E}\!\left[\|\mathbf{e}_k^t\|^2\mid\mathbf{w}^{t}\right].
\label{eq:app-one-step}
\end{align}
Summing over $t = 0,\ldots,T-1$, taking total expectation,
using $F(\mathbf{w}^T)\geq F^*$, and dividing by $\eta ET/4$ yields
\begin{align}
&\frac{1}{T}\sum_{t=0}^{T-1}\mathbb{E}\|\nabla F(\mathbf{w}^{t})\|^2
\leq \frac{4(F(\mathbf{w}^{0})-F^*)}{\eta ET}
  + 4L\eta E\sigma^2 \notag \\
&  \quad + 8L\eta E\delta^2
  + \frac{5}{\eta^2E^2T}
    \sum_{t=0}^{T-1}\sum_{k=1}^{K}\frac{n_k}{n}\mathbb{E}\|\mathbf{e}_k^t\|^2.
\label{eq:app-final}
\end{align}
Finally, if $\|\mathbf{e}_k^t\|^2 \leq 2\sum_{\ell\in\mathcal{L}_k}
\varepsilon_{k,\ell}^2(r_k^{(\ell)})\|G_k^{(\ell),t}\|_F^2$, the last term is explicitly controlled by the layer-wise MPS reconstruction errors.
\end{proof}

Theorem~\ref{thm:convergence} shows that MESH-FL preserves the standard nonconvex FL convergence pattern, with the additional term caused by compression. Thus, rank allocation should reduce the average reconstruction-error term while satisfying each client's payload budget.

\subsection{Optimality of Entropy-Guided Allocation}\label{sec:optimality}
The convergence bound identifies compression error as the additional optimization penalty. We therefore choose ranks by minimizing the entropy-weighted surrogate induced by Proposition~\ref{prop:entropy-mps} under the per-client payload constraint.

\begin{theorem}[Optimal monotone allocation]\label{thm:allocation}
Fix client \(k\) and let \(\psi_k^{(\ell)}:=\Psi_{k,\ell}\). Consider the surrogate problem in \eqref{eq:psi-surrogate} with payload model
\begin{align}
\phi_\ell(r)=a_\ell r+b_\ell r^2,\quad a_\ell,b_\ell\geq0.
\label{eq:quadratic-payload}
\end{align}
The problem is strictly convex and admits a unique optimizer. Moreover, the optimal rank \(r_k^{(\ell)*}\) is strictly increasing in \(\psi_k^{(\ell)}\). Under the linearized payload model \(b_\ell=0\), the optimizer is
\begin{align}
r_k^{(\ell)*}=\frac{C_k\sqrt{\psi_k^{(\ell)}/a_\ell}}{\sum_{\ell'\in\mathcal{L}_k}\sqrt{a_{\ell'}\psi_k^{(\ell')}}}.
\label{eq:closed-form}
\end{align}
Under the exact quadratic payload model, the optimizer satisfies
\begin{align}
\frac{\psi_k^{(\ell)}}{(r_k^{(\ell)*})^2}=\lambda_k\left(a_\ell+2b_\ell r_k^{(\ell)*}\right),
\label{eq:kkt-allocation}
\end{align}
where \(\lambda_k>0\) enforces the payload constraint with equality. 
\end{theorem}
\begin{proof}
The objective
\begin{align}
\sum_{\ell\in\mathcal{L}_k}\frac{\psi_k^{(\ell)}}{r_k^{(\ell)}}
\label{eq:app-allocation-objective}
\end{align}
is strictly convex on \(r_k^{(\ell)}>0\), because each term \(\psi_k^{(\ell)}/r_k^{(\ell)}\) has second derivative \(2\psi_k^{(\ell)}/(r_k^{(\ell)})^3>0\). The feasible set
\begin{align}
\left\{\{r_k^{(\ell)}\}:\sum_{\ell\in\mathcal{L}_k}\phi_\ell\!\left(r_k^{(\ell)}\right)\leq C_k,\;r_k^{(\ell)}>0\right\}
\label{eq:app-allocation-feasible-set}
\end{align}
is convex because each \(\phi_\ell(r)=a_\ell r+b_\ell r^2\) is convex for \(a_\ell,b_\ell\geq0\). Hence, the surrogate allocation problem is strictly convex and admits a unique optimizer.

The Lagrangian is
\begin{align}
\mathcal{J}
=\sum_{\ell\in\mathcal{L}_k}\frac{\psi_k^{(\ell)}}{r_k^{(\ell)}}
+\lambda_k\left(\sum_{\ell\in\mathcal{L}_k}\phi_\ell\!\left(r_k^{(\ell)}\right)-C_k\right),
\label{eq:app-lagrangian}
\end{align}
with \(\lambda_k\geq0\). Since the objective is strictly decreasing in each \(r_k^{(\ell)}\), the budget constraint is active at the optimum, so \(\lambda_k>0\). The KKT stationarity condition for each layer \(\ell\) is
\begin{align}
-\frac{\psi_k^{(\ell)}}{(r_k^{(\ell)})^2}
+\lambda_k\left(a_\ell+2b_\ell r_k^{(\ell)}\right)=0.
\label{eq:app-kkt-allocation}
\end{align}
For fixed \(\lambda_k>0\), define
\begin{align}
f_\ell(r)=-\frac{\psi_k^{(\ell)}}{r^2}+\lambda_k(a_\ell+2b_\ell r).
\label{eq:app-root-function}
\end{align}
The function \(f_\ell\) is strictly increasing and continuous on \(r>0\), with \(f_\ell(r)\to-\infty\) as \(r\to0^+\). If \(b_\ell>0\), then \(f_\ell(r)\to+\infty\) as \(r\to\infty\); if \(b_\ell=0\), then \(f_\ell(r)\to\lambda_ka_\ell>0\) as \(r\to\infty\). Therefore, there exists a unique positive root \(r_k^{(\ell)*}\) satisfying \eqref{eq:app-kkt-allocation}. For fixed \(\lambda_k\), increasing \(\psi_k^{(\ell)}\) shifts \(f_\ell\) downward, so the unique root must increase to restore equality. Hence, \(r_k^{(\ell)*}\) is strictly increasing in \(\psi_k^{(\ell)}\).

When \(b_\ell=0\), \eqref{eq:app-kkt-allocation} reduces to
\begin{align}
-\frac{\psi_k^{(\ell)}}{(r_k^{(\ell)})^2}+\lambda_ka_\ell=0,
\label{eq:app-linear-kkt}
\end{align}
which gives
\begin{align}
r_k^{(\ell)*}=\sqrt{\frac{\psi_k^{(\ell)}}{\lambda_ka_\ell}}.
\label{eq:app-linear-rank}
\end{align}
Substituting \eqref{eq:app-linear-rank} into the active budget constraint gives
\begin{align}
\sum_{\ell\in\mathcal{L}_k}a_\ell r_k^{(\ell)*}
=\frac{1}{\sqrt{\lambda_k}}\sum_{\ell\in\mathcal{L}_k}\sqrt{a_\ell\psi_k^{(\ell)}}
=C_k.
\label{eq:app-budget-substitution}
\end{align}
Therefore,
\begin{align}
\frac{1}{\sqrt{\lambda_k}}
=\frac{C_k}{\sum_{\ell\in\mathcal{L}_k}\sqrt{a_\ell\psi_k^{(\ell)}}}.
\label{eq:app-lambda-solution}
\end{align}
Substituting \eqref{eq:app-lambda-solution} into \eqref{eq:app-linear-rank} gives
\begin{align}
r_k^{(\ell)*}
=\frac{C_k\sqrt{\psi_k^{(\ell)}/a_\ell}}{\sum_{\ell'\in\mathcal{L}_k}\sqrt{a_{\ell'}\psi_k^{(\ell')}}}.
\label{eq:app-closed-form-allocation}
\end{align}
Under the quadratic payload model \(b_\ell>0\), the optimizer is characterized by the KKT condition \eqref{eq:app-kkt-allocation}, equivalently
\begin{align}
\frac{\psi_k^{(\ell)}}{(r_k^{(\ell)*})^2}
=\lambda_k\left(a_\ell+2b_\ell r_k^{(\ell)*}\right),
\label{eq:app-quadratic-kkt-final}
\end{align}
where \(\lambda_k>0\) is chosen so that the payload constraint holds with equality.
\end{proof}
Theorem~\ref{thm:allocation} establishes the monotone allocation structure used by MESH-FL: for fixed payload coefficients, layers with larger spectral-complexity weight receive larger rank under the client budget. The practical rule \(r_k^{(\ell)}\propto e^{\widehat{H}_k^{(\ell)}/2}\) is obtained by using \(e^{\widehat{H}_k^{(\ell)}}\) as a low-cost online proxy for the spectral-complexity ordering in \eqref{eq:closed-form}, followed by budget scaling via bisection and integer feasibility correction.


\section{Experimental Results}
\label{sec:experiments}
\subsection{Experimental Setup}

The experiments were conducted on a 15-node heterogeneous cluster consisting of 5 Raspberry Pi 4 Model B (quad-core Cortex-A72 at 1.8 GHz, 8 GB LPDDR4, 64 GB storage, batch size 16) and 10 Raspberry Pi 5 (quad-core Cortex-A76 at 2.4 GHz, 16 GB LPDDR4X, 128 GB storage, batch size 32), all connected to a central server via a local Ethernet switch. Pi 4 nodes were assigned higher compression targets than Pi 5 nodes to reflect their tighter uplink capacity.

We used AV-MNIST, a dataset combining MNIST images~\cite{lecun1998mnist} with spoken-digit recordings from FSDD~\cite{jakobovski2018fsdd}. Each sample consists of a $28 \times 28$ grayscale image and a 1,000-dimensional MFCC feature vector (20 coefficients, 8 kHz, hop length 160). The dataset includes 60,000 training and 10,000 test pairs across 10 classes. The dataset is partitioned non-IID using a Dirichlet distribution with concentration parameter \(\alpha=0.1\), following the protocol introduced in~\cite{hsu2019measuring} for controlling label-distribution skew in FL.

Each client was randomly assigned a modality profile (image-only, audio-only, or multimodal), resulting in heterogeneous participation patterns across training rounds. Image-only clients updated only the image encoder, audio-only clients updated the audio encoder, and multimodal clients updated both encoders and the shared fusion module.

\begin{figure*}[t]
  \centering
  \includegraphics[width=0.95\textwidth]{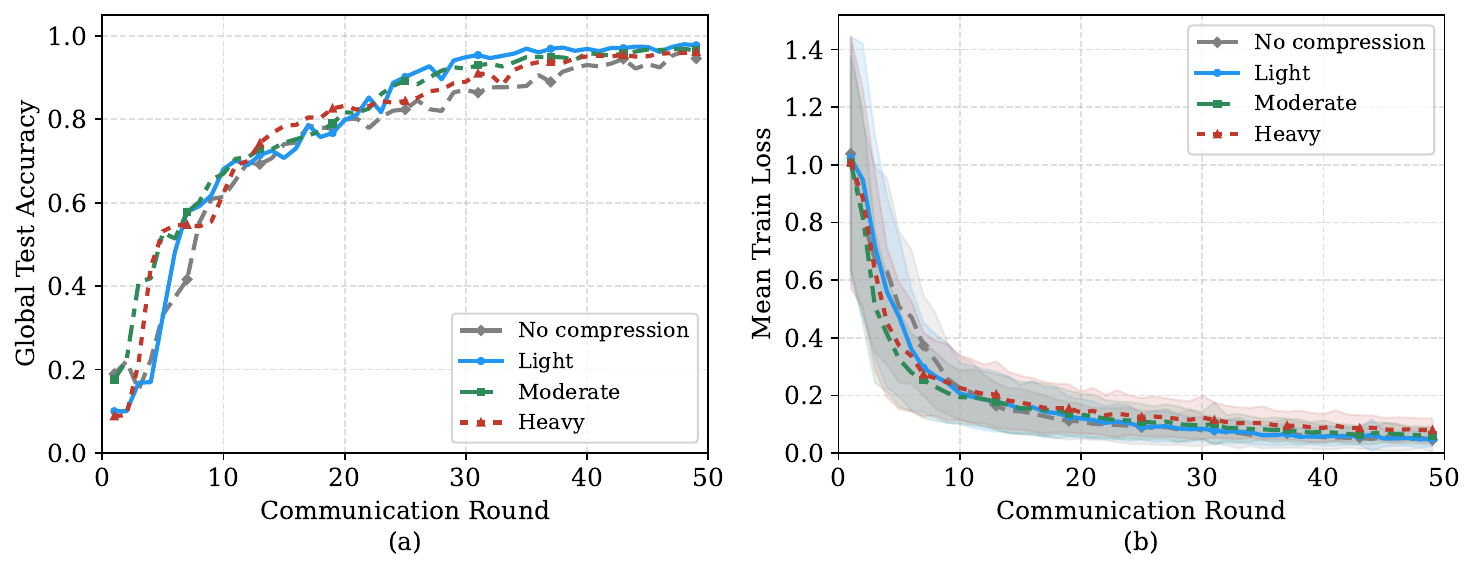}
  \caption{Global test accuracy and mean train loss ($\pm$ std across
    clients) over 50 rounds on AV-MNIST ($\alpha=0.1$).}
  \label{fig:convergence}
\end{figure*}
The global model consisted of a CNN image encoder, an MLP audio encoder, and a shared two-layer fusion module producing a 10-class output. The experiments ran for $T = 50$ rounds with full participation and $E = 5$ local SGD steps (lr = 0.01). Three compression levels were evaluated: light compression targeting $2$–$10\times$ for Pi 5 and $5$–$10\times$ for Pi 4; moderate compression targeting $30$–$50\times$ for Pi 5 and $40$–$60\times$ for Pi 4; and heavy compression targeting $40$–$60\times$ for Pi 5 and $60$–$80\times$ for Pi 4. These levels produced fleet-averaged compression ratios of $9.36\times$, $40.52\times$, and $56.82\times$, respectively. The entropy truncation parameter was fixed at $q = 10$. Baselines included standard FedAvg~\cite{mcmahan2017communication} without compression, TopK sparsification ($s = 1\%$) with error feedback~\cite{lin2018dgc}, QSGD with 4-bit quantization~\cite{alistarh2017qsgd}, and PowerSGD with rank 4~\cite{vogels2019powersgd}.
\subsection{Communication Efficiency and Learning Performance}

Fig.~\ref{fig:convergence} demonstrates that all three MESH-FL compression levels converge faster and achieve higher final accuracy than the uncompressed FedAvg baseline. Specifically, light compression reaches 97.73\%, surpassing FedAvg (95.72\%) by 2.01\%, while moderate and heavy compression reach 96.47\% and 96.20\%, respectively. This result stems from the interaction between entropy-guided MPS compression and the non-IID data distribution ($\alpha = 0.1$). By projecting each layer-wise update onto its dominant spectral subspace, MESH-FL filters high-frequency, client-specific gradient components amplified in non-IID settings, acting as implicit gradient regularization. The loss curves confirm this: FedAvg shows larger cross-client variance throughout training, while compressed variants maintain tighter agreement from round 10 onward.
\begin{figure}[t]
  \centering
  \includegraphics[width=0.95\columnwidth]{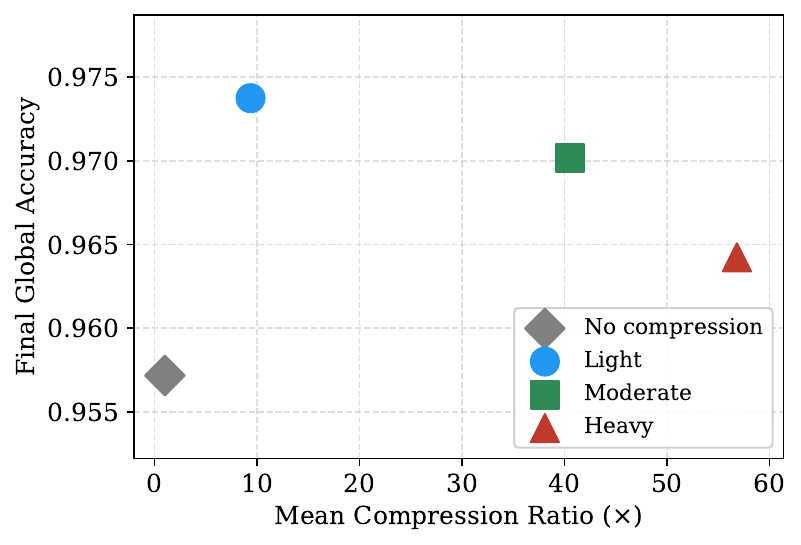}
  \caption{Final global accuracy and mean compression ratio.}
  \label{fig:tradeoff}
\end{figure}

The accuracy-compression tradeoff is illustrated in Fig.~\ref{fig:tradeoff}. Notably, the uncompressed baseline ($1.0\times$) provides the lowest accuracy (95.72\%), while every MESH-FL compression level improves both accuracy and communication efficiency simultaneously. Light compression achieves the best accuracy (97.73\%) at $9.36\times$ compression, while moderate and heavy compression trade a modest accuracy reduction (1.26 and 1.53\% below light) for substantially higher compression ratios ($40.52\times$ and $56.82\times$).
\begin{figure}[t]
  \centering
  \includegraphics[width=0.9\columnwidth]{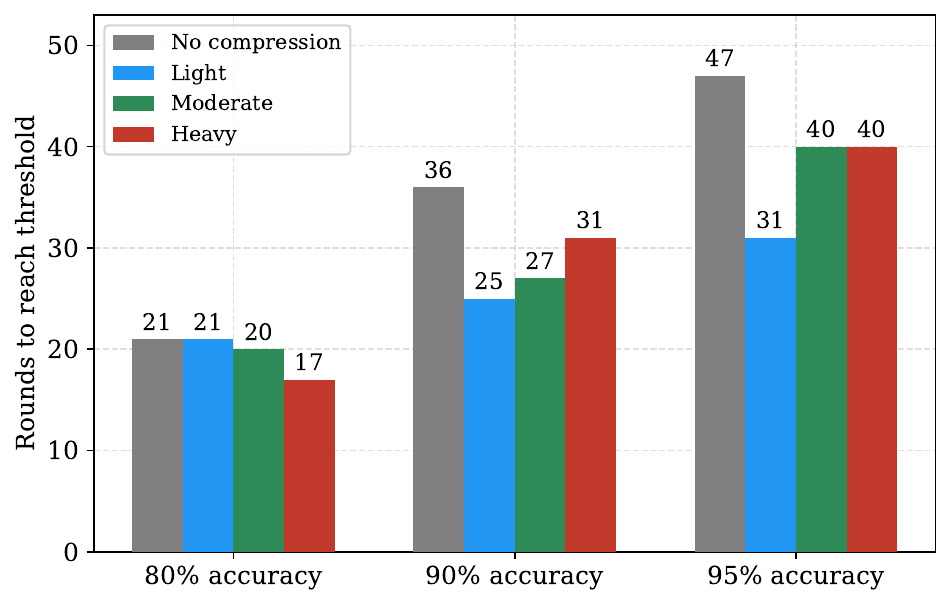}
  \caption{Communication rounds required to reach 80\%, 90\%, and 95\% accuracy thresholds under each compression level.}
  \label{fig:rounds}
\end{figure}
Fig.~\ref{fig:rounds} shows that all compressed variants reach accuracy thresholds faster than FedAvg. At 95\% accuracy, FedAvg requires 47 rounds versus 31, 40, and 40 for light, moderate, and heavy compression. The 16-round reduction for light compression, combined with its $9.36\times$ lower per-round payload, translates into a dramatic reduction in total bytes transmitted to reach convergence, as shown in Fig.~\ref{fig:savings}(b).
\begin{figure*}[t]
  \centering
  \includegraphics[width=0.99\textwidth]{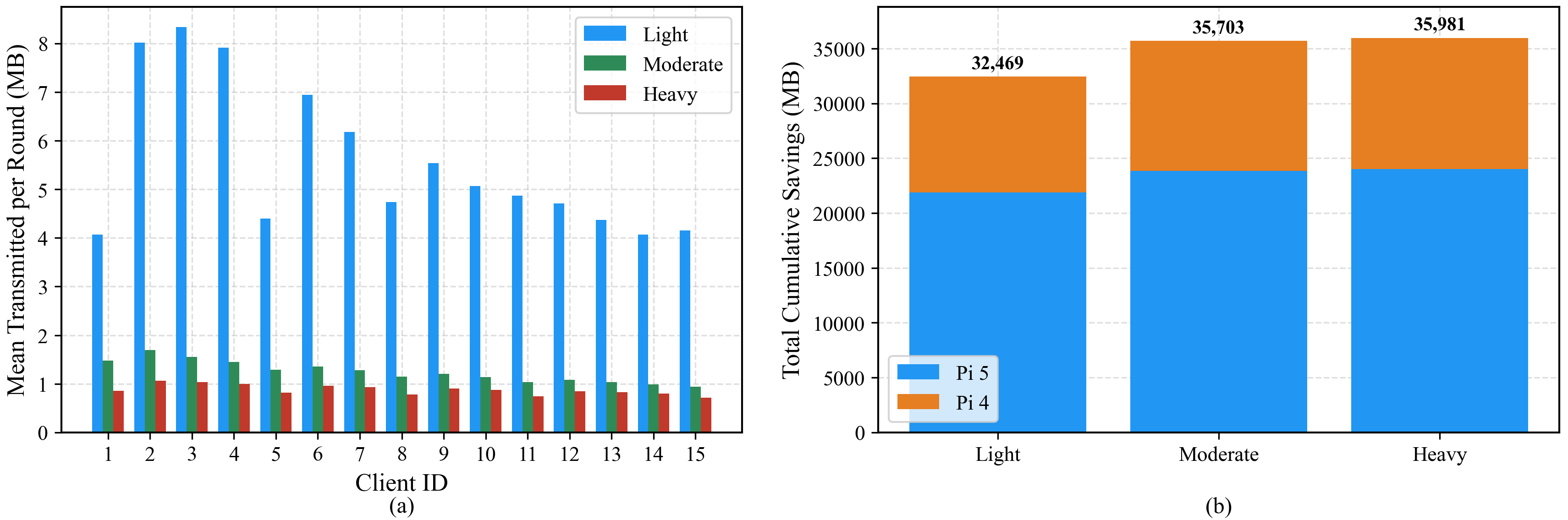}
  \caption{Per-client mean transmitted data per round and total
    cumulative savings over 50 rounds, broken down by device type.}
  \label{fig:savings}
\end{figure*}
Fig.~\ref{fig:savings} illustrates per-client payload and cumulative savings across compression levels. Multimodal clients transmit more data than unimodal clients due to their larger active layer sets. Moderate and heavy compression reduce per-client payload to 1.25~MB and 0.88~MB per round, respectively, compared with 48.85~MB under no compression. Heavy compression saves approximately 35{,}900~MB ($\approx 36$~GB) relative to the uncompressed baseline across all 15 clients over 50 rounds, with Pi~5 nodes contributing the larger share due to their higher modality participation rate.
\subsection{Comparison with Baselines}

\begin{table}[t]
\renewcommand{\arraystretch}{1.3}
\caption{Accuracy and compression on AV-MNIST ($\alpha=0.1$).}
\label{tab:accuracy}
\centering
\begin{tabular}{l|c|c|c}
\hline
\textbf{Method} & \textbf{Acc.\ (\%)} & \textbf{Ratio ($\times$)} & \textbf{Rnd@95\%} \\
\hline
FedAvg~\cite{mcmahan2017communication}       & 95.72 & 1.00  & 47 \\
TopK ($s{=}1\%$)~\cite{lin2018dgc}           & 93.4  & 10.2  & 48 \\
QSGD (4-bit)~\cite{alistarh2017qsgd}         & 92.1  & 8.0   & ${>}50$ \\
PowerSGD ($r{=}4$)~\cite{vogels2019powersgd} & 95.8  & 12.1  & 43 \\
\hline
MESH-FL Light                                & 97.73 & 9.36  & 31 \\
MESH-FL Moderate                             & 96.47          & 40.52 & 40 \\
MESH-FL Heavy                                & 96.20          & 56.82 & 40 \\
\hline
\end{tabular}
\end{table}

\begin{table}[t]
\caption{System efficiency on AV-MNIST; per-round values averaged
  across 15 clients.}
\label{tab:efficiency}
\renewcommand{\arraystretch}{1.4}
\begin{minipage}{\columnwidth}
\centering
\resizebox{\columnwidth}{!}{%
\begin{tabular}{l|c|c|c|c}
\hline
\textbf{Method} & \textbf{MB/rnd} & \textbf{Comp.\ (s)} & \textbf{Trans.\ (s)} & \textbf{Total TX (GB)} \\
\hline
FedAvg      & 48.85 & 0.01 & 4.89 & 33.63 \\
TopK ($s{=}1\%$)       & 4.79  & 0.02 & 0.48 & 3.37  \\
QSGD (4-bit)        & 6.11  & 0.02 & 0.61 & 4.54  \\
PowerSGD ($r{=}4$)& 4.04  & 0.18 & 0.40 & 2.54  \\
\hline
MESH-FL Light                                & 5.56  & 0.22 & 0.56 & 2.52  \\
MESH-FL Moderate                             & 1.25  & 0.31 & 0.12 & \textbf{0.73}  \\
MESH-FL Heavy                                & 0.88  & 0.38 & 0.09 & \textbf{0.51}  \\
\hline
\end{tabular}}
\end{minipage}
\end{table}

Table~\ref{tab:accuracy} compares accuracy and compression efficiency across all methods. MESH-FL light achieves the highest accuracy (97.73\%) at $9.36\times$ compression, outperforming even the uncompressed FedAvg baseline. PowerSGD is the strongest baseline at 95.8\% accuracy, but reaches only $12.1\times$ compression and still falls 1.93\% below MESH-FL light. Both QSGD (92.1\%) and TopK (93.4\%) underperform FedAvg under the strongly non-IID regime: QSGD suffers from quantization noise that compounds with gradient dissimilarity, while TopK transmits only 1\% of gradient coordinates per round, disrupting layer-wise gradient structure across heterogeneous modality branches.

\begin{figure*}[b]
  \centering
  \includegraphics[width=0.99\textwidth]{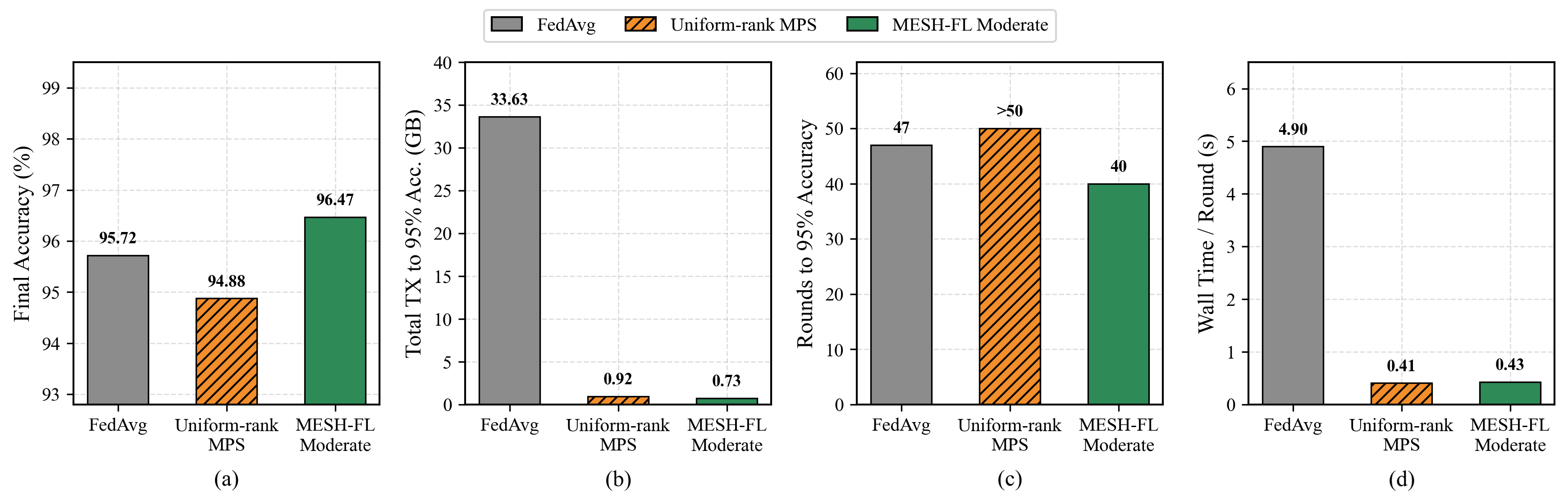}
  \caption{Entropy-guided and uniform-rank MPS under identical moderate budgets.}
  \label{fig:ablation}
\end{figure*}

Table~\ref{tab:efficiency} provides the full runtime breakdown. The Comp.\ column captures client-side SVD and MPS overhead, consistent with the $\mathcal{O}(\sum_\ell m_\ell n_\ell q)$ complexity of Section~\ref{sec:system}-F. Although MESH-FL adds 0.22--0.38~s of computation per round compared to the near-zero overhead of TopK and QSGD, it more than compensates through transmission savings. The net per-round wall time (Comp.\ $+$ Trans.) is 0.78~s, 0.43~s, and 0.47~s for light, moderate, and heavy compression, all well below FedAvg (4.90~s) and below TopK (0.50~s) and QSGD (0.63~s) at the moderate and heavy levels. Most critically, the total data transmitted (Total TX) column shows that MESH-FL heavy requires only 0.51~GB to reach 95\% accuracy versus 33.63~GB for FedAvg ($66\times$ reduction) and 2.54~GB for the next-best PowerSGD ($5\times$ reduction).

\subsection{Ablation Study}

To isolate the impact of entropy-guided rank allocation, Fig.~\ref{fig:ablation} compares MESH-FL moderate with a uniform-rank MPS baseline under identical per-client payload budgets. MESH-FL reaches 96.47\% final accuracy, outperforming uniform-rank MPS (94.88\%) by 1.59\% and FedAvg (95.72\%) by 0.75\%, despite using the same $40.52\times$ compression budget. It also reaches 95\% accuracy within 40 rounds. In contrast, uniform-rank MPS fails to reach this threshold within 50 rounds, reducing Total TX to 0.73~GB compared with at least 0.92~GB for uniform-rank MPS and 33.63~GB for FedAvg. Since both MPS variants have nearly identical per-round wall-clock times, namely 0.43~s for MESH-FL and 0.41~s for uniform-rank MPS, the improvement cannot be attributed to computational differences. The gain is most visible during rounds 5--20 under $\alpha=0.1$, where entropy guidance allocates higher ranks to diffuse fusion layers spectrally and fewer ranks to low-entropy encoder layers. This ablation confirms that the gain results from more effective layer-wise rank allocation, not additional communication or computation.

\section{Conclusion}\label{sec:conclusion}

In this paper, we presented MESH-FL, an entropy-guided MPS update-compression framework for FL under joint modality and device heterogeneity. By projecting layer-wise updates onto dominant spectral subspaces, MESH-FL allocates compression ranks according to spectral entropy, solving a convex surrogate rank-allocation problem while preserving the standard non-convex FL convergence rate. Experiments on a 15-node Raspberry Pi 4/5 cluster showed up to $56.8\times$ compression, up to a 2.01\% accuracy improvement over uncompressed FedAvg under strongly non-IID data ($\alpha = 0.1$), and up to a $66\times$ reduction in transmitted data to convergence. These results demonstrate that entropy-guided rank allocation can reduce communication cost while maintaining, and in some cases improving, model accuracy. The main limitations are that the exponential tail-decay assumption has not been verified across all layers and training regimes, and the observed accuracy gain lacks a formal theoretical explanation. Future work will extend the three-core MPS design to higher-order tensor-train decompositions, combine entropy-guided compression with asynchronous aggregation, and study when low-rank gradient compression improves generalization in non-IID FL.


\bibliographystyle{IEEEtran}
\bibliography{refs}

\end{document}